\documentclass[letterpaper, 10pt, conference]{ieeeconf}

\IEEEoverridecommandlockouts 

\usepackage{graphicx}
\usepackage{times}
\usepackage{amsmath}
\usepackage{amssymb}
\usepackage{bm}
\graphicspath{{./Figures/}}
\usepackage[listofformat=subsimple, caption=false,font=footnotesize]{subfig}

\usepackage{braket}

\usepackage{booktabs,ragged2e}
\usepackage[flushleft]{threeparttable}

\newtheorem{theorem}{Theorem}
\newtheorem{lemma}[theorem]{Lemma}

\usepackage{xcolor}

\newcommand{\blue}[1]{\textcolor{blue}{#1}}

\definecolor{orange}{rgb}{1.0, 0.22, 0.0}

\title{\LARGE \bf
Motion and Force Planning for Manipulating Heavy Objects by Pivoting
}

\author{Amin Fakhari, Aditya Patankar, and Nilanjan Chakraborty
\thanks{The authors are with the Department of Mechanical Engineering, Stony Brook University, Stony Brook, NY 11794, USA, {\tt\small \{amin.fakhari, aditya.patankar, nilanjan.chakraborty\}@stonybrook.edu}.}
}

\begin{document}

\maketitle
\thispagestyle{empty}
\pagestyle{empty}

\begin{abstract}
Manipulation of objects by exploiting their contact with the environment can enhance both the dexterity and payload capability of robotic manipulators. A common way to manipulate heavy objects beyond the payload capability of a robot is to use a sequence of pivoting motions, wherein, an object is moved while some contact points between the object and a support surface are kept fixed. The goal of this paper is to develop an algorithmic approach for automated plan generation for object manipulation with a sequence of pivoting motions. A plan for manipulating a heavy object consists of a sequence of joint angles of the manipulator, the corresponding object poses, as well as the joint torques required to move the object. The constraint of maintaining object contact with the ground during manipulation results in nonlinear constraints in the configuration space of the robot, which is challenging for motion planning algorithms. Exploiting the fact that pivoting motion corresponds to movements in a subgroup of the group of rigid body motions, $SE(3)$, we present a novel task-space based planning approach for computing a motion plan for both the manipulator and the object while satisfying contact constraints. We also combine our motion planning algorithm with a grasping force synthesis algorithm to ensure that friction constraints at the contacts and actuator torque constraints are satisfied. We present simulation results with a dual-armed Baxter robot to demonstrate our approach.

\end{abstract}


\section{Introduction}
\label{sec:Introduction}
Manipulation of heavy and bulky objects is a challenging task for  manipulators and humanoid robots. An object is considered heavy if the manipulator's joint torques are not large enough to balance the object weight while lifting it off the ground. Thus, heavy objects cannot be manipulated with usual pick-and-place strategy due to actuator saturation.
Consider the manipulation scenario shown in Fig.~\ref{Fig:Motivation}, where a heavy object has to be moved from an initial pose $\mathcal{C}_O$ to a final pose $\mathcal{C}_F$ by a dual-armed robot. The object has to negotiate a step during the manipulation which implies that the final pose cannot be achieved by either pick-and-place strategies or by pushing. One possible way to move the object and negotiate the step is to use a sequence of pivoting motions, which we call object gaiting, and this is a common strategy used by humans to manipulate heavy objects. Therefore, the goal of this paper is to develop an algorithmic approach to compute a plan for manipulating heavy objects by a sequence of pivoting motions.

\begin{figure}[t]
    \centering
    \includegraphics[scale=0.42]{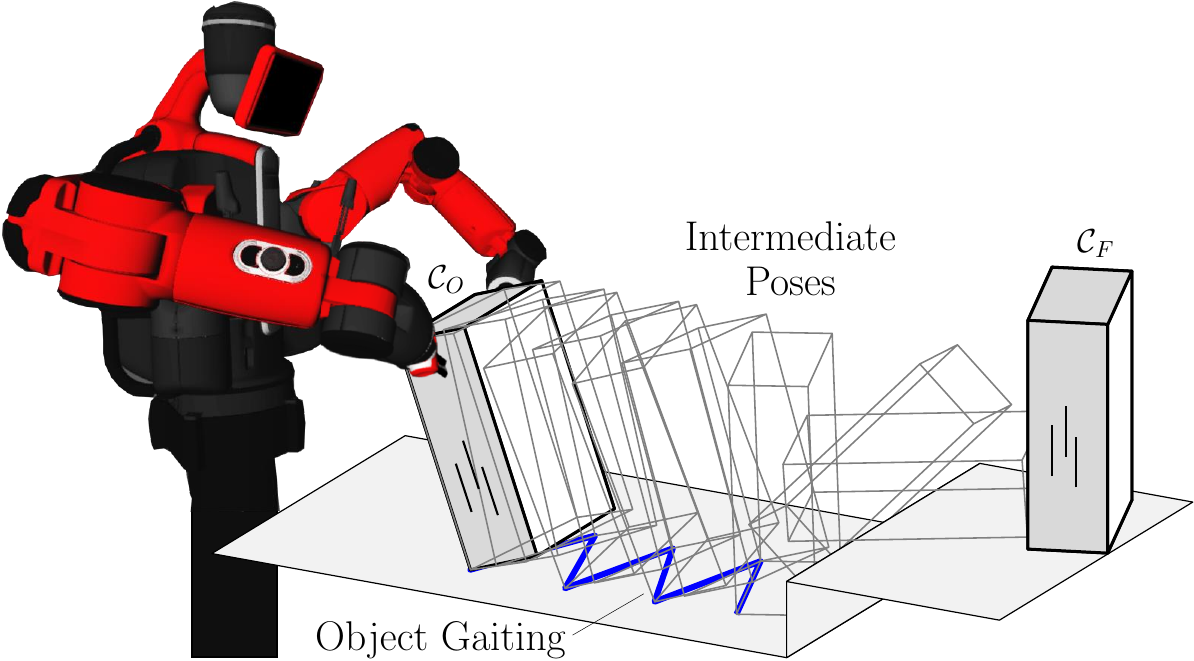}
    \caption{Schematic sketch of dual-handed manipulation of a heavy object between two given poses $\mathcal{C}_O$ and $\mathcal{C}_F$ by a sequence of pivoting motions.}
\label{Fig:Motivation}
\end{figure}

In a pivoting motion, we move the object while creating a sequence of point or line contacts with the environment (see Figure~\ref{Fig:Motivation}). A point contact acts like a spherical joint, whereas a line contact acts like a revolute joint. The location and axes of these joints change during a gaiting motion. These joints are force-closed joints and can only be implemented through adequate frictional force at the object-ground contact that prevents slippage. Because of the making and breaking of the contacts the equations of motion of the object change and technically the system of equations form a hybrid dynamical system. Planning of motion through intermittent contact considering the switching dynamics is difficult and computationally costly, in general, although some attempts have been made in this direction(e.g.,~\cite{PosaCT14}). To reduce the complexity and cost of planning, one approach is to decouple the overall planning problem and first compute a kinematic motion plan and then compute force inputs that follow the kinematic motion plan and satisfy the dynamics constraints. The challenge in the decoupled approach is to ensure that the kinematic plans are dynamically feasible. 

In this paper, we take the decoupled approach to motion planning for pivoting. 
Thus, a plan for pivoting operations consists of (a) {\em Motion plan}: a sequence of joint angles of the manipulators and the corresponding object poses that maintains contact with the ground and ensures that the contacts between the hand and the object are maintained  (b) {\em Force plan}: a sequence of joint torques that are within the actuator limits and ensure that there is enough force at the object ground contact to prevent slippage. Furthermore, to ensure that the kinematic plan is dynamically feasible, we also want to ensure that the manipulator does not lose the grasp of the object and there is no slippage at the hand-object contact. In this paper, we will focus on the motion planning problem. We have studied the problem of computing the force plan (or force synthesis problem), for a given motion plan, in~\cite{Patankar2020}, and we will combine it with our motion plan to generate torques to achieve the motion. 

The key challenge in solving the motion planning problem is that the kinematic constraints of the object maintaining a spherical or a revolute joint with the ground during the motion corresponds to nonlinear manifold constraints in the joint space of the manipulator. For example, in the scenario shown in Fig.~\ref{Fig:Motivation}, the configuration space is $20$ dimensional and since the position kinematics equations are nonlinear, the pivoting constraints will form nonlinear manifolds in this $20$ dimensional space. In sampling-based motion planning in joint space ($\mathbb{J}$-space), these constraints are hard to deal with, although there have been some efforts in this direction~\cite{BerensonSK11,JailletP12,Stilman10,KimU16,YaoK07, bonilla2015sample, KingstonMK2019}. Furthermore, in manipulation by gaiting, where we are performing a sequence of pivoting operations, these manifold constraints are not known beforehand since they depend on the choice of the pivot points (or lines) which has to be computed as a part of the plan. In this paper, we present a novel task-space ($\mathbb{T}$-space) based approach for generating the motion plan that exploits the fact that the kinematic constraints of a revolute or spherical joint constrains the motion of the object to a subgroup of $SE(3)$.

{\bf Contributions}: We present a two-step approach for computing the motion plan. In the first step, we develop an algorithm to compute a sequence of intermediate poses for the object to go from the initial to the goal pose. Two consecutive intermediate poses implicitly determine a point or line on the object and the ground that stay fixed during motion, thus encoding motion about a revolute or a spherical joint. In the second step, we use Screw Linear Interpolation (ScLERP) to determine a task space path between two intermediate poses, along with resolved motion rate control (RMRC)~\cite{whitney1969resolved,Pieper68} to convert the task space path to a joint space path. The advantage of using ScLERP is that it (a) automatically satisfies the kinematic motion constraints that the contact points between the object and the ground co not change during the pivoting motion without explicitly encoding it (b) the object does not penetrate the ground (c) the relative pose between the fingers (or manipulators) at the grasping points do not change, i.e., the closed chain constraints formed by the manipulator and object is maintained throughout the whole motion without explicitly encoding it.
Thus, the joint space path that we compute along with the object path automatically ensures that the kinematic contact constraints are satisfied and the object remains grasped. This computationally efficient approach for motion planning for manipulation by pivoting, where apparently complicated constraints can be satisfied without explicitly modeling them, is the key contribution of this paper. Assuming that the inertial force are negligible, i.e., the motion is quasi-static, we also show that our motion plan can be combined with the second order cone programming (SOCP) based approach to compute joint torques and grasping forces~\cite{Patankar2020}, while ensuring that all no-slip constraints at the contacts and actuator limits are satisfied. Thus, we can ensure that the kinematic plan is feasible when one considers the forces and actuator limits. We demonstrate our approach in simulation using a dual-armed Baxter robot.   

\section{Related Work}
The use of external environment contacts to enhance the in-hand manipulation capability was first studied by Chavan-Dafle in \cite{Dafle2014}. More recently Hou \textit{et. al} \cite{HouMason2018} have developed a planning algorithm for quasi-static reorientation of 3D objects on a table using a parallel-jaw gripper by having 3D mesh model of the objects. However, they are using a gripper with a special mechanism to allow pivoting at the gripper-object contact and compliant rolling at the object-table contact. In another research, they have referred to the use of environment contact as \textit{shared grasping} wherein they treat the environment as an additional finger \cite{HouMason2020}. They have provided stability analysis of shared grasping by using \textit{Hybrid Force-Velocity Control} (HFVC).
Murooka \textit{et. al.} \cite{Murooka2015} proposed a method for pushing a heavy object by an arbitrary region of a humanoid robot. Polverini \textit{et. al.} \cite{Polverini2020} also developed a control architecture for a humanoid robot which is able to exploit the complexity of the environment to perform the pushing task of a heavy object.
Pivoting was first was first introduced by Aiyama \textit{et. al.} \cite{aiyama1993pivoting} as a new method of graspless/non-prehensile manipulation. Based on this method, Yoshida \textit{et. al.} \cite{yoshida2007pivoting,yoshida2008whole,yoshida2010pivoting} developed a whole-body motion planner for a humanoid robot to autonomously plan a pivoting
strategy for manipulating bulky objects. They first planned a sequence of collision-free Reeds and Shepp paths (especially straight and circular paths in $\mathbb{R}^2$), then convert these paths into a sequence of pivoting motions. However, this method is limited to the motion on Reeds and Shepp curves to satisfy a nonholonomic constraint, which is not always required. Thus, it is not a general, efficient, and optimum way to manipulate objects between two given poses, especially when there are no obstacles in the workspace.
Hence, we propose a general gait planning method as an optimization problem by defining the \textit{intermediate poses} and using the ScLERP to manipulate the object by gaiting between any two arbitrary poses.
Although there are task-space based force control algorithms~\cite{bouyarmane2018quadratic,khatib1995inertial}, they do not consider planning through the hybrid dynamics created by the intermittent contact. Furthermore~\cite{bouyarmane2018quadratic} approximates the second order friction cone as a polyhedral cone, which can lead to infeasible solutions even when feasible plans exist. 

\section{Preliminaries}

\noindent
\textbf{Quaternions and Rotations}: The quaternions are the set of hypercomplex numbers, $\mathbb{H}$. A quaternion $Q \in \mathbb{H}$ can be represented as a 4-tuple $Q = (q_0, \boldsymbol{q}_r) = (q_0, q_1, q_2, q_3)$, $q_0 \in \mathbb{R}$ is the real scalar part, 
$\boldsymbol{q}_r=(q_1, q_2, q_3) \in \mathbb{R}^3$ corresponds to the imaginary part.
The conjugate, norm, and inverse of a quaternion $Q$ is given by
$Q^* = (q_0, -\boldsymbol{q}_r)$, $\lVert Q \rVert = \sqrt{Q Q^*} = \sqrt{Q^* Q}$,
and $Q^{-1} = Q^*/{\lVert Q \rVert}^2$, respectively. Addition and multiplication of two quaternions
$P = (p_0, \boldsymbol{p}_r)$ and
$Q = (q_0, \boldsymbol{q}_r)$ are performed as $P+Q = (p_0 + q_0, \boldsymbol{p}_r + \boldsymbol{q}_r)$ and $PQ = (p_0 q_0 - \boldsymbol{p}_r \cdot \boldsymbol{q}_r, p_0 \boldsymbol{q}_r + q_0 \boldsymbol{p}_r + \boldsymbol{p}_r \times \boldsymbol{q}_r)$.
The quaternion $Q$ is a \textit{unit quaternion}
if ${\lVert Q \rVert} = 1$, and consequently, $Q^{-1} = Q^*$. Unit quaternions are used to represent the set of all rigid body rotations,  $SO(3)$, the Special Orthogonal group of dimension $3$. Mathematically,  $SO(3)=\left\{\boldsymbol{R} \in \mathbb{R}^{3 \times 3}\left|\boldsymbol{R}^{\mathrm{T}} \boldsymbol{R}=\boldsymbol{R} \boldsymbol{R}^{\boldsymbol{T}}=\boldsymbol{I}_3,\right| \boldsymbol{R} \mid=1\right\}$, where $\boldsymbol{I}_3$ is a $3\times3$ identity matrix and $\left| \cdot \right|$ is the determinant operator. The unit quaternion corresponding to a rotation is $Q_R = (\cos\frac{\theta}{2}, \boldsymbol{l} \sin\frac{\theta}{2})$, where $\theta \in [0,\pi]$ is the angle of rotation about a unit axis $\boldsymbol{l} \in \mathbb{R}^3$. 

\noindent
\textbf{Dual Quaternions and Rigid Displacements}:
In general, dual numbers are defined as $d = a + \epsilon b$ where $a$ and $b$ are elements of an algebraic field, and $\epsilon$ is a \textit{dual unit} with $\epsilon ^ 2 = 0, \epsilon \ne 0$.
Similarly, a dual quaternion $D$ is defined as $D= P + \epsilon Q$
where $P, Q \in \mathbb{H}$. The conjugate, norm, and inverse of the dual quaternion $D$ is represented as $D^* = P^* + \epsilon Q^*$, $\lVert D \rVert = \sqrt{D D^*} = \sqrt{P P^* + \epsilon (PQ^* + QP^*)}$, and $D^{-1} = D^*/{\lVert D \rVert}^2$,
respectively. Another definition for the conjugate of $D$ is represented as $D^\dag = P^* - \epsilon Q^*$. Addition and multiplication of two dual quaternions $D_1= P_1 + \epsilon Q_1$ and $D_2= P_2 + \epsilon Q_2$ are performed as $D_1 + D_2 = (P_1 + P_2) + \epsilon (Q_1 + Q_2)$ and $D_1 D_2 = (P_1 P_2) + \epsilon (P_1 Q_2 + Q_1 P_2) $.
The dual quaternion $D$ is a \textit{unit dual quaternion} if ${\lVert D \rVert} = 1$, i.e., ${\lVert P \rVert} = 1$ and $PQ^* + QP^* = 0$, and consequently, $D^{-1} = D^*$. Unit dual quaternions can be used to represent the group of rigid body displacements, $SE(3) = \mathbb{R}^3 \times SO(3)$, $S E(3)=\left\{(\boldsymbol{R}, \boldsymbol{p}) \mid \boldsymbol{R} \in S O(3), \boldsymbol{p} \in \mathbb{R}^{3}\right\}$. An element $\boldsymbol{T} \in SE(3)$, which is a pose of the rigid body, can also be expressed by a $4 \times 4$ homogeneous transformation matrix as
$\boldsymbol{T} = \left[\begin{smallmatrix}\boldsymbol{R}&\boldsymbol{p}\\\boldsymbol{0}&1\end{smallmatrix}\right]$ where $\boldsymbol{0}$ is a $1 \times 3$ zero vector. A rigid body displacement (or transformation) is represented by a unit dual quaternion $D_T = Q_R + \frac{\epsilon}{2} Q_p Q_R$ where $Q_R$ is the unit quaternion corresponding to rotation and  $Q_p = (0, \boldsymbol{p}) \in \mathbb{H}$ corresponds to the translation.


\noindent
\textbf{Screw Displacement}: Chasles-Mozzi theorem states that the general Euclidean displacement/motion of a rigid body from the origin $\boldsymbol{I}$ to $\boldsymbol{T} = (\boldsymbol{R},\boldsymbol{p}) \in SE(3)$
can be expressed as a rotation $\theta$ about a fixed axis $\mathcal{S}$, called the \textit{screw axis}, and a translation $d$ along that axis (see Fig.~\ref{Fig:ScrewDisplacement}). Plücker coordinates can be used to represent the screw axis by $\boldsymbol{l}$ and $\boldsymbol{m}$, where $\boldsymbol{l} \in \mathbb{R}^3$ is a unit vector that represents the direction of the screw axis $\mathcal{S}$, $\boldsymbol{m} = \boldsymbol{r} \times \boldsymbol{l}$, and $\boldsymbol{r} \in \mathbb{R}^3$ is an arbitrary point on the axis. Thus, the screw parameters are defined as $\boldsymbol{l}, \boldsymbol{m}, \theta, d$.
The screw displacements can be expressed by the dual quaternions as $D_T = Q_R + \frac{\epsilon}{2} Q_p Q_R = (\cos \frac{\Phi}{2}, L \sin \frac{\Phi}{2})$ where $\Phi = \theta + \epsilon d$ is a dual number and $L = \boldsymbol{l} + \epsilon \boldsymbol{m}$ is a
dual vector.
A power of the dual quaternion $D_T$ is then defined as $D_T^{\tau} = (\cos \frac{\tau \Phi}{2}, L \sin \frac{\tau \Phi}{2})$, $\tau >0$.


\begin{figure}[!htbp]
    \centering
    \includegraphics[scale=0.57]{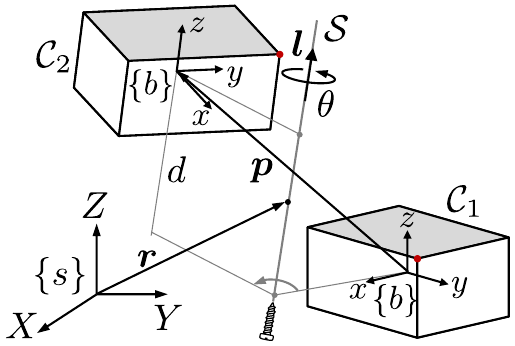}
    \caption{Screw displacement from pose $\mathcal{C}_1$ to pose $\mathcal{C}_2$.}
\label{Fig:ScrewDisplacement}
\end{figure}

\noindent
\textbf{Screw Linear Interpolation (ScLERP)}: To perform a one degree-of-freedom smooth screw motion (with a constant rotation and translation rate) between two object poses in $SE(3)$, the screw linear interpolation (ScLERP) can be used. The ScLERP provides a \textit{straight line} in $SE(3)$ which is the closest path between two given poses in $SE(3)$. 
If the poses are represented by unit dual quaternions $D_{1}$ and $D_{2}$, the path provided by the ScLERP is derived by $D(\tau) = D_1 (D_1^{-1}D_2)^{\tau}$ where $ \tau \in[0,1]$ is a scalar path parameter. 
As $\tau$ increases from 0 to 1, the object moves between two poses along the path
$D(\tau)$ by the rotation $\tau \theta$ and translation $\tau d$. Let $D_{12} = D_1^{-1}D_2$. To compute $D_{12}^\tau$, the screw coordinates $\boldsymbol{l}, \boldsymbol{m}, \theta, d$ are first extracted from $D_{12} = P + \epsilon Q = (p_0,\boldsymbol{p}_r) + \epsilon (q_0,\boldsymbol{q}_r) = (\cos\frac{\theta}{2}, \boldsymbol{l} \sin\frac{\theta}{2}) + \epsilon Q$ by $\boldsymbol{l} = \boldsymbol{p}_r/ \lVert \boldsymbol{p}_r \lVert $, $\theta = 2 \, \mathrm{atan2}(\lVert \boldsymbol{p}_r \lVert, p_0)$, $d = \boldsymbol{p} \cdot \boldsymbol{l}$, and $\boldsymbol{m} = \frac{1}{2} (\boldsymbol{p} \times \boldsymbol{l} + (\boldsymbol{p}-d \boldsymbol{l})\cot \frac{\theta}{2})$ where $\boldsymbol{p}$ is derived from $2QP^* = (0, \boldsymbol{p})$ and $\mathrm{atan2}(\cdot)$ is the two-argument arctangent. Then, $D_{12}^\tau = (\cos \frac{\tau \Phi}{2}, L \sin \frac{\tau \Phi}{2})$ is directly derived from $\left(\cos \frac{\tau \theta}{2}, \sin \frac{\tau \theta}{2}\boldsymbol{l}\right)+\epsilon \left( -\frac{\tau d}{2}\sin \frac{\tau \theta}{2}, \frac{\tau d}{2}\cos \frac{\tau \theta}{2}\boldsymbol{l}+\sin \frac{\tau \theta}{2}\boldsymbol{m} \right) $. Note that $\theta =0, \pi$ corresponds to pure translation between two poses and the screw axis is at infinity.

\section{Problem Statement}
\label{sec:ProblemStatement}
Let us assume that we want to manipulate a heavy cuboid object quasi-statically by using $n$ manipulators, while maintaining contact with environment, from an initial pose $\mathcal{C}_O \in SE(3)$ to a final pose $\mathcal{C}_F \in SE(3)$. We assume that both $\mathcal{C}_O$ and $\mathcal{C}_F$ are in the robot's workspace.
Let $\boldsymbol{\Theta}^i = [\theta_1^i, \theta_2^i, \cdots, \theta_{l_i}^i] \in \mathbb{R}^{l_i}$ be the vector of joint angles of the $i$-th $l_i$-DoF manipulator, which represents the \textit{joint space} ($\mathbb{J}$-space) or the \textit{configuration space} ($\mathbb{C}$-space) of the manipulator.
Moreover, $\mathcal{E}^i \in SE(3)$ is defined as the pose of the end-effector of the $i$-th manipulator where $\mathcal{E}^i = \mathcal{FK}(\boldsymbol{\Theta}^i)$ and $\mathcal{FK}(\cdot)$ is the manipulator forward kinematics map. Therefore, $\boldsymbol{\Theta}_O^i \in \mathbb{R}^{l_i}$ and $\mathcal{E}_O^i \in SE(3)$ represent the initial configuration of the $i$-th manipulator (in $\mathbb{J}$-space) and pose of $i$-th end-effector, respectively, corresponding to the object initial pose $\mathcal{C}_O$ and $\boldsymbol{\Theta}_F^i \in \mathbb{R}^{l_i}$ and $\mathcal{E}_F^i \in SE(3)$ represent the final configuration of the $i$-th manipulator (in $\mathbb{J}$-space) and pose of $i$-th end-effector, respectively, corresponding to the object final pose $\mathcal{C}_F$. We assume that the position of the manipulator-object contact $c_i$ is given and the transformation between the frames $\{e_i\}$ and $\{c_i\}$ remains constant during the manipulation, i.e., there is no relative motion (slippage) at the contact interface. 

Our motion planning problem is now defined as computing a sequence of joint angles $\boldsymbol{\Theta}^i(j)$, where $j=1,\cdots,m$, $\boldsymbol{\Theta}^i(1) = \boldsymbol{\Theta}^i_O$, $\boldsymbol{\Theta}^i(m) = \boldsymbol{\Theta}^i_F$, to manipulate the object while maintaining contact with the environment from its initial pose $\mathcal{C}_O$ to a final pose $\mathcal{C}_F$ when $(\mathcal{C}_O, \mathcal{E}_O^i, \boldsymbol{\Theta}_O^i)$ and $(\mathcal{C}_F, \mathcal{E}_F^i)$ ($i=1,...,n$) are given. Moreover, our force planning problem is computing the minimum contact wrenches required to be applied at $c_i$ during the object manipulation to balance the external wrenches (e.g., gravity) and also the environment contact wrenches using the method we have presented in \cite{Patankar2020}. 



\begin{figure}[!h]
    \centering
    \includegraphics[scale=0.45]{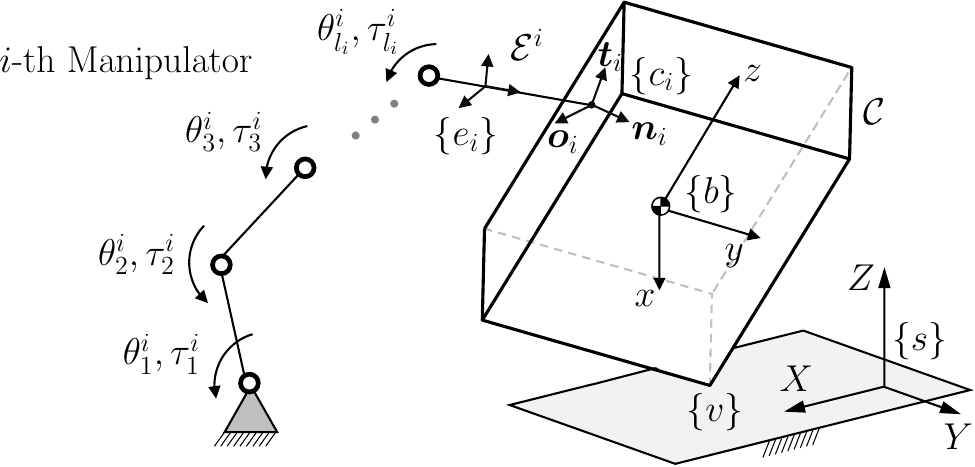}
    \caption{An cuboid-shape object being tilted at one of its vertices.}
\label{Fig:Cube_Manipulator}
\end{figure}

\textbf{Solution Approach Overview}: Generally speaking, to move an object while maintaining contact we can use two primitive motions, namely, (1) \textit{sliding} on a vertex, edge, or face of the object in contact with the environment (Fig.~\ref{Fig:SRP}-\subref{Fig:SRP_S}) and
(2) \textit{pivoting} about an axis passing through a vertex, edge, or face of the object in contact with the environment (Fig.~\ref{Fig:SRP}-\subref{Fig:SRP_T},\subref{Fig:SRP_P}, Fig.~\ref{Fig:Motivation}). All other motions can be made by combining these primitive motions. Note that we consider \textit{tumbling} as a special case of pivoting when the axis of rotation passes through an object edge or face. Manipulation by sliding (or pushing) can be useful in many scenarios like picking a penny off a table. However, in heavy and bulky object manipulation scenarios, sliding may not give feasible solutions. Thus, in this paper, we will focus on manipulation using the pivoting primitive.


Our \textit{manipulation strategy} can be described briefly as follows.
(i) Given the initial and final pose of the object, we first determine if multiple pivoting moves have to be made and, if necessary, compute intermediate poses of the object. (ii) Using the dual quaternion representation of these poses, we compute  paths in $SE(3)$ using the ScLERP for the object and end-effectors. These paths automatically satisfies all the basic task related constraints (without any additional explicit representation of the constraints).
(iii) We use the (weighted) pseudoinverse of the Jacobian to derive the joint angles in the $\mathbb{J}$-space from the computed $\mathbb{T}$-space path. (iv) Finally, we compute the minimum required contact wrenches and manipulators' joint torques required for object manipulation. Note that the steps (ii) to (iv) can be done either sequentially or they can be interleaved in a single discrete time-step 


    
    
    

\section{Force Planning}
\label{sec:force_plan}

In this section, we briefly review our force planning algorithm proposed in \cite{Patankar2020}. Consider an object which is in contact with the environment at $m$ contact positions and being manipulated quasi-statically by $n$ $l_i$-DOF manipulators ($i = 1,...,n$) an $n$ contact positions (Fig.~\ref{Fig:Cube_Manipulator}). We define the grasping force optimization  problem (GFOP) as a convex optimization problem as
\begin{subequations}
\begin{align}
& {\underset {\bm{F}_{C}, \bm{F}_{E}, \bm{\tau}}{\text{minimize} }} & & \max(\bm{F}_{n}) \label{eq:ObjectiveFunction}\\
& \text{subject to} & & {{\bm{\mathrm{G}}}_{C} \bm{F}_{C}} + {{\bm{\mathrm{G}}}_{E} \bm{F}_{E}} + \bm{f}_{\mathrm{ext}} = \bm{0}, \label{eq:NetWrench}\\
&&& \bm{F}_{C} \in {\mathcal{K}_C}, \label{eq:FC_manipulator}\\
&&& \bm{F}_{E} \in {\mathcal{K}_E}, \label{eq:FC_environment}\\
&&& \bm{\tau} + {\bm{\mathrm{J}}}^\mathrm{T} \bm{F}_{C} - \bm{\tau}_g = \bm{0}, \label{eq:manipulatorTorque}\\
&&& \bm{\tau}_{\rm min} \leq \bm{\tau} \leq \bm{\tau}_{\rm max} \label{eq:manipulatorTorqueLimits}.
\end{align}
\label{equation:optimization}
\end{subequations}
The objective function \eqref{eq:ObjectiveFunction} is the maximum normal force at the object-manipulator contacts, where ${\bm F}_{n} \in \mathbb{R}^n$ is the vector of normal contact forces. Equation~\eqref{eq:NetWrench} is the equilibrium constraint where ${{\bm{\mathrm{G}}}_{C} \bm{F}_{C}} \in \mathbb{R}^6$ and ${{\bm{\mathrm{G}}}_{E} \bm{F}_{E}} \in \mathbb{R}^6$ are the total wrenches exerted by the manipulators and environment through the contacts, with respect to the body frame $\{b\}$, respectively, $\bm{F}_C \in \mathbb{R}^{6n}$ and $\bm{F}_E \in \mathbb{R}^{6m}$ are the vectors containing all the contact wrenches at the manipulators and environment contacts, respectively, ${{\bm{\mathrm{G}}}_{C}}$ is the grasp matrix and ${{\bm{\mathrm{G}}}_{E}}$ is the matrix that converts the environment contact wrench to the body frame, and $\bm{f}_{\mathrm{ext}} \in \mathbb{R}^6$ is the total external wrench (including the object weight) applied to the object at the body frame \{$b$\}. \eqref{eq:FC_manipulator} and \eqref{eq:FC_environment} represent the friction cone constraints at all the manipulators and environment contacts, respectively, where ${\mathcal{K}_C}$ and ${\mathcal{K}_E}$ are the second-order cones (SOC).
\eqref{eq:manipulatorTorqueLimits} represents the manipulator joint torque constraints, where $\bm{\tau} \in \mathbb{R}^{l}$ is a vector containing the joint torques of all the manipulators and computing from \eqref{eq:manipulatorTorque}, $\bm{\tau}_g \in \mathbb{R}^{l}$ is the vector of joint torques due to gravity, and ${\bm{\mathrm{J}}} = \mathrm{diag}({\bm{\mathrm{J}}}_1, \dots, {\bm{\mathrm{J}}}_n) \in \mathbb{R}^{6n \times l}$ is the overall Jacobian matrix of the manipulators ($l=\sum_{i=1}^n{l_i}$).


\section{Pivoting}
Pivoting is a motion where an object is moved while maintaining a point or line contact with a support surface. When an object maintains a point contact, the constraints on motions are same as those imposed by a spherical joint. Thus, the motion of the object is restricted to $SO(3)$, which is a subgroup of $SE(3)$, and the axis of rotation passes through the contact point. During pivoting with line contact (or tumbling), the constraint on the motion is same as that imposed by a revolute joint with the axis of the joint being the line of contact. Thus, in this case, the motion of the object is restricted to $SO(2)$, which is also a subgroup of $SE(3)$. This mathematical structure of pivoting motions is key to our ScLERP approach as we discuss below.

Suppose an object can reach a goal pose from a start pose using a single pivoting motion. This can happen when the start and the goal poses are such that there is a common vertex, say $v$, between the start and goal poses that lie on the support surface (Fig.~\ref{Fig:SLERP_Linear}, Fig.~\ref{Fig:SRP}-\subref{Fig:SRP_P}). To plan the motion of the object between the start and goal pose via interpolation such that the contact at the vertex $v$ is maintained and the object does not penetrate the surface one should be careful about the interpolation scheme. 

Two popular choices for interpolation between two given poses $(Q_1,\bm{p}_1)$ and $(Q_2,\bm{p}_2)$ (using unit quaternion parameterization of orientation) are (a) linear interpolation for both orientation and position, i.e., $Q(\tau) = [{Q_1 + (Q_2 - Q_1)\tau}]/{\|Q_1 + (Q_2 - Q_1)\tau\|}$, $ \bm{p}(\tau) = \bm{p}_1 + (\bm{p}_2-\bm{p}_1)\tau $ (Fig.~\ref{Fig:SLERP_Linear}-\subref{Fig:Linear}), and (b) spherical linear interpolation (SLERP) for orientation and linear interpolation for position, i.e., $ Q(\tau) = Q_1 (Q_1^{-1}Q_2)^{\tau} $, $ \bm{p}(\tau) = \bm{p}_1 + (\bm{p}_2-\bm{p}_1)\tau $, $ \tau \in[0,1]$ (Fig.~\ref{Fig:SLERP_Linear}-\subref{Fig:SLERP}).
If we use these linear interpolations between the end poses in the space of parameters, the object will penetrate the support surface, so the motion plan will not be feasible. The motion obtained will also change with the choice of the coordinate frames for the initial and final pose. The advantage of using ScLERP is that it is coordinate invariant. Furthermore, since the pivoting motions also belongs to a subgroup of $SE(3)$, ScLERP ensures that all the intermediate poses will lie in the same subgroup that contains the initial and goal pose (i.e., all intermediate poses will have the vertex $v$ fixed to the support surface). Thus, for motion planning, it is not necessary to explicitly enforce the pivoting constraint. Lemma~\ref{lemma:FixedPoint} formalizes this discussion. Furthermore, ScLERP results in a motion along the shortest path, which in the pivoting scenario leads to a rotation about the axis through the pivot point(s) by the smallest angle. Thus, the non-penetration constraint of the object with the ground is also satisfied without explicitly enforcing it.

\begin{figure}[!htbp]
    \centering
    \subfloat[]{\includegraphics[scale=0.36]{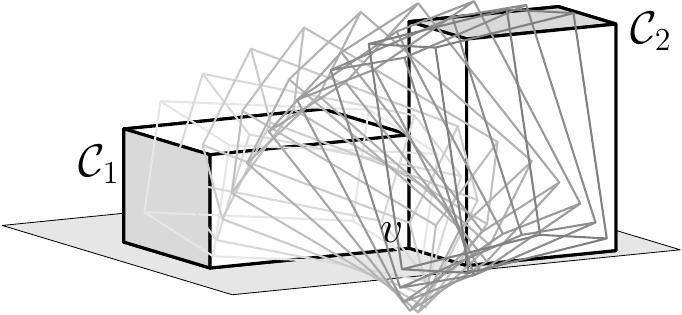}\label{Fig:Linear}}
    \subfloat[]{\includegraphics[scale=0.36]{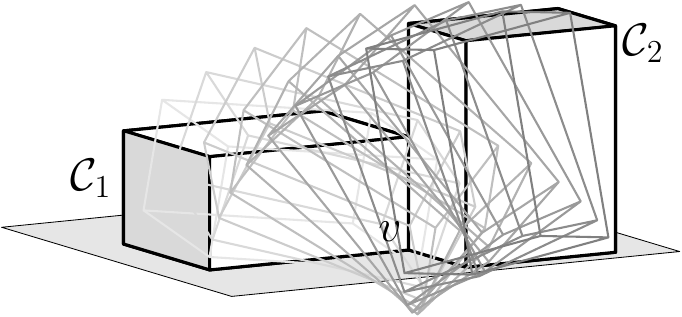}\label{Fig:SLERP}}
    \caption{Interpolation between two given poses $\mathcal{C}_1$ and $\mathcal{C}_2$, (a) linear interpolation for both orientation and position, (b) spherical linear interpolation (SLERP) for orientation and linear interpolation for position.}
\label{Fig:SLERP_Linear}
\end{figure}

\begin{figure}[!htbp]
    \centering
    \subfloat[]{\includegraphics[scale=0.36]{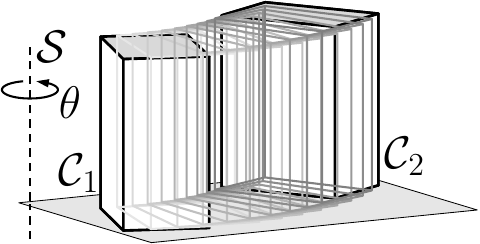}\label{Fig:SRP_S}}
    \subfloat[]{\includegraphics[scale=0.36]{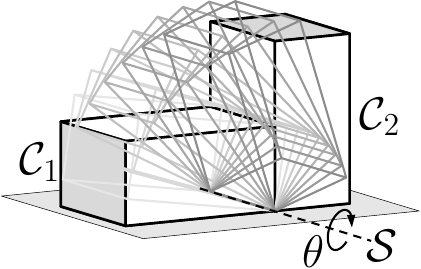}\label{Fig:SRP_T}}
    \subfloat[]{\includegraphics[scale=0.36]{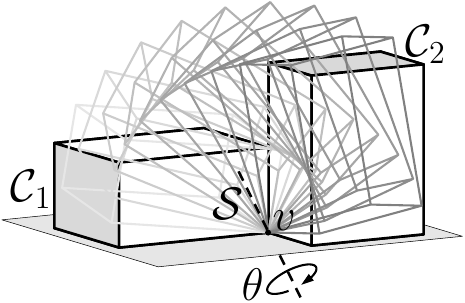}\label{Fig:SRP_P}}
    \caption{Examples of the primitive motions for manipulating cuboid-shape objects by exploiting the environment contact, (a) sliding or pushing on a face, (b) pivoting about an edge (tumbling), (c) pivoting about a vertex.}
\label{Fig:SRP}
\end{figure}

\begin{lemma}
Let $D_1 = Q_{R1} + \frac{\epsilon}{2} Q_{p1} Q_{R1}$ and $D_2 = Q_{R2} + \frac{\epsilon}{2} Q_{p2} Q_{R2}$ be two unit dual quaternions representing two poses of a rigid body. If a point $\boldsymbol{v} \in \mathbb{R}^3$ in the rigid body has the same position in both poses, the position of this point remains the same in all the poses provided by the ScLERP $D(\tau) = D_1 (D_1^{-1}D_2)^{\tau}$ where $ \tau \in[0,1]$.
\label{lemma:FixedPoint}
\end{lemma}
\begin{proof}
Let $Q_v = (0,\boldsymbol{v}) \in \mathbb{H}$ be a pure quaternion representing the point $\boldsymbol{v}$. Since the point $\boldsymbol{v}$ has the same position in both poses $D_1$ and $D_2$, therefore
\begin{align}
D_1(1+\epsilon Q_v)D_1^\dag & =  D_2(1+\epsilon Q_v)D_2^\dag, \\
\therefore  \,\, Q_{p2} - Q_{p1} & = Q_{R1} Q_v Q_{R1}^*  - Q_{R2} Q_v Q_{R2}^*.
\label{eq:Qp2_Qp1}
\end{align}
Therefore, the transformation from $D_1$ to $D_2$ is derived as
\begin{equation}
\begin{split}
        D_{12} & = D_1^{*}D_2 = Q_{R1}^* Q_{R2} + \frac{\epsilon}{2}Q_{R1}^* (Q_{p2} - Q_{p1}) Q_{R2}\\
        &= Q_{R1}^* Q_{R2} + \frac{\epsilon}{2}(Q_v Q_{R1}^* Q_{R2} - Q_{R1}^* Q_{R2} Q_v).
\end{split}
\label{eq:D_1D_2_}
\end{equation}
By representing the rotation $Q_{R1}^* Q_{R2}$ as $(\cos\frac{\theta}{2}, \boldsymbol{l} \sin\frac{\theta}{2}) \in \mathbb{H}$ (where $\boldsymbol{l}$ is a unit vector along the  screw axis and $\theta$ is rotation about the screw axis), ($\ref{eq:D_1D_2_}$) can be simplified as
\begin{equation}
    D_{12} = (\cos\frac{\theta}{2}, \boldsymbol{l}\sin\frac{\theta}{2}) + \epsilon (0, \boldsymbol{v} \times \boldsymbol{l} \sin\frac{\theta}{2}) = P + \epsilon Q.
\label{eq:D_12}
\end{equation}
The translation $d$ along the screw axis is determined by $d = \boldsymbol{p} \cdot \boldsymbol{l}$ where $\boldsymbol{p}$ is derived from $2QP^* = (0, \boldsymbol{p})$. By using (\ref{eq:D_12}),
\begin{equation}
\boldsymbol{p} = \boldsymbol{v} \times \boldsymbol{l} \sin\frac{\theta}{2} \cos\frac{\theta}{2} - (\boldsymbol{v} \times \boldsymbol{l}) \times \boldsymbol{l} \sin^2\frac{\theta}{2},
\label{eq:p}
\end{equation}
and $d = \boldsymbol{p} \cdot \boldsymbol{l} = 0$. Therefore, the transformation $D(\tau)$ is a pure rotation about the fixed point $\boldsymbol{v}$ on the screw axis.
\end{proof}

Furthermore, {\em when using multiple manipulators to pivot an object and assuming that there is no relative motion at the hand-object contact, the motion of each end-effector can be obtained independently by ScLERP using a shared interpolation parameter. This will ensure that the  constraint that the relative end-effector poses of the manipulators are unchanged during motion is maintained  without explicitly encoding it (this follows from Lemma $3$ of~\cite{Sarker2020} and so we do not repeat the formal statements and proofs here)}. In other words, it is guaranteed that the manipulators and the object always form a closed chain during the entire motion without explicitly encoding the constraint. In the next section, we use pivoting as a primitive motion for motion planning between any two given poses in $\mathbb{T}$-space.

\section{Motion Planning in Task Space}
\label{sec:MotionPlanningTS}
To manipulate a polyhedral object between any two given poses $\mathcal{C}_O$ and $\mathcal{C}_F$ while maintaining contact with the environment, multiple pivoting moves can be combined by defining a set of appropriate \textit{intermediate poses}. The set of the intermediate poses $\mathcal{C}_I = \{\mathcal{C}_I^1,\mathcal{C}_I^2, \cdots, \mathcal{C}_I^h \}$ are defined in a way that the motion between any two successive poses $\{\mathcal{C}_O, \mathcal{C}_I, \mathcal{C}_F \}$ can be represented by a single constant screw pivoting move. 
Thus, we can conveniently represent the motion between any two given object poses $\mathcal{C}_O$ and $\mathcal{C}_F$ in $SE(3)$ by using ScLERP to ensure that the object maintains its contact with the environment continuously. The object manipulation strategies on a flat surface can be categorized into 3 cases; (\textbf{Case I}) If $\mathcal{C}_O$ and $\mathcal{C}_F$ have a contact edge or vertex in common, the final pose can be achieved by pivoting the object about the common point or edge (Fig.~\ref{Fig:SRP}-\subref{Fig:SRP_T},\subref{Fig:SRP_P}).
(\textbf{Case II}) If $\mathcal{C}_O$ and $\mathcal{C}_F$ do not have any edge or vertex in common but the same face of the object is in contact with the environment in both poses, different strategies can be considered. One of the strategies is using a sequence of pivoting motions about the object edges (tumbling). In this motion, the travel distance is discrete and depends on the object size and it may not be suitable for manipulating some objects like furniture.
In this situation, we can manipulate the object is \textit{object gaiting} (Fig.~\ref{Fig:IntermediateConfigs_Edges}-a) which is defined as a sequence of pivoting motions on two adjacent object vertices in contact
(see \ref{subsec:IntermediatePosesObjectGaiting} and \ref{subsec:GaitPlanning}).
(\textbf{Case III}) If
the adjacent or opposite faces of the object are in contact with the environment in both poses, a combination of pivoting and gaiting is required to achieve the final pose as shown in Fig.~\ref{Fig:Examples}. Depending on the manipulators' physical limitations, object gaiting is more efficient only when a specific face of the object is in contact with the environment. For instance, manipulation on the longer edge of the cuboid shown in Fig.~\ref{Fig:IntermediateConfigs_Edges}-a may be more difficult than two other edges.

\begin{figure}[!htbp]
    \centering
    \subfloat[]{\includegraphics[scale=0.26]{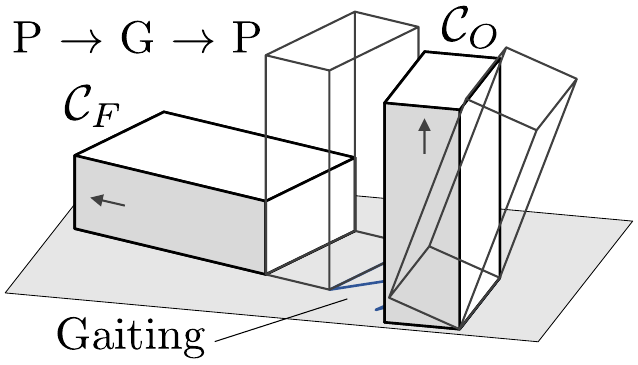}}
    \subfloat[]{\includegraphics[scale=0.26]{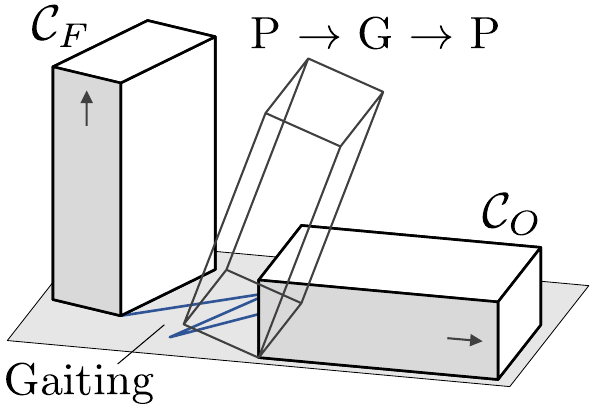}}
    \subfloat[]{\includegraphics[scale=0.26]{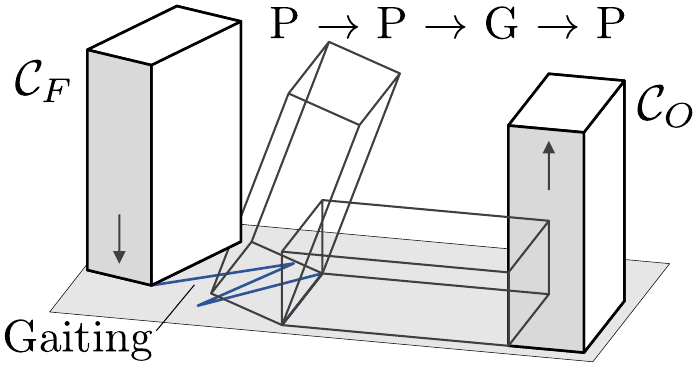}} 
    \caption{Examples of the object manipulation with primitive motions when two adjacent (a,b) or opposite (c) object faces are in contact with the environment in initial and final poses (P: Pivoting, G: Gaiting).}
\label{Fig:Examples}
\end{figure}

\subsection{Intermediate Poses in Object Gaiting}
\label{subsec:IntermediatePosesObjectGaiting}
Let us assume that the axes of the body frame $\{b\}$ are parallel to the cuboid edges and the inertia frame $\{s\}$ is attached to the supporting plane such that $Z$-axis is perpendicular to the plane (Fig.~\ref{Fig:ObjectGaiting}). Three successive intermediate poses while pivoting about the vertex $a$ are shown in Fig.~\ref{Fig:ObjectGaiting}-a,b. The object is initially in the pose $\mathcal{C}_I^1 = (R_{1}, p_{1})$ (Fig.~\ref{Fig:ObjectGaiting}-a) holding on the contact edge $ab$. The pose $\mathcal{C}_I^2 = (R_{2}, p_{2})$ (Fig.~\ref{Fig:ObjectGaiting}-a) is achieved by rotating the object by a small angle $\beta$ along the edge passing through the vertex $a$; therefore, $R_{2} = R_{1} R_{x}({-\beta})$ and only the vertex $a$ is in the contact. Finally, the pose $\mathcal{C}_I^3 = (R_{3}, p_{3})$ (Fig.~\ref{Fig:ObjectGaiting}-b) is determined by rotating $\mathcal{C}_I^1$ by an angle $\alpha$ along $Z$-axis about the vertex $a$; therefore, $R_{3} = R_Z({\alpha}) R_{1}$ and the edge $ab$ is again in contact with the environment. This procedure can be also repeated for the vertex $b$. By using ScLERP between these intermediate poses, we can obtain a smooth motion for object gaiting while maintaining contact with the environment. Note that the angles $\gamma$ and $\beta$ can be chosen arbitrarily and we do so in this paper. However, if there are secondary objectives like avoiding small obstacles on the ground while pivoting, $\gamma$ and $\beta$ can be chosen to satisfy those objectives.



\subsection{Gait Planning}
\label{subsec:GaitPlanning}
To manipulate the object from an initial pose $\mathcal{C}_O$ to a final pose $\mathcal{C}_F$ by object gaiting, a sequence of the rotation angle $\alpha$ between these two poses should be properly determined (Fig.~\ref{Fig:IntermediateConfigs_Edges}-a). Let $k$ be the number of required edge contacts and $\bm{\alpha} = [\alpha_1, \cdots, \alpha_k] \in \mathbb{R}^k$ be the angles between the contact edges as shown in Fig.~\ref{Fig:IntermediateConfigs_Edges}-b. We can find $\bm{\alpha}$ using an optimization problem as
\begin{equation}
\begin{aligned}
&{\underset {\bm{\alpha}}{\operatorname {minimize}}}&&  \lVert \bm{\alpha} \rVert \\[-8pt]  
&\operatorname {subject\;to} && \boldsymbol{x} = \pm w \sum_{i=1}^k{\left( -1 \right) ^{i}\left[ \begin{array}{@{\mkern0mu} c @{\mkern0mu}}
	\cos \left( \alpha _O \pm \bar{\alpha} \right)\\
	\sin \left( \alpha _O \pm \bar{\alpha} \right)\\
\end{array} \right]},\\[-5pt]
&&& \alpha _{F} - \alpha _O  = \pm \sum_{i=1}^k{\left( -1 \right) ^{i}\alpha_i },\\
&&& \left| \alpha _i \right| \leq \alpha_{\text{max}},\ \ i=1,...,k,
\end{aligned}
\label{eq:GaitPlanning}
\end{equation}
where $\bar{\alpha} = \sum_{j=1}^i{\left( -1 \right) ^{j}\alpha _j }$, $\alpha_{\text{max}}$ is the maximum allowed rotation angle, $w$ is the length of the edge contact, and $\alpha_{O}$ and $\alpha_{F}$ represent the orientation of the contact edges $a_O b_O$ and $a_F b_F$ relative to $X$-axis, respectively. The negative sign correspond to the case that the first gait begins from the edge $a_O$, where $\boldsymbol{x} = \boldsymbol{b}_{F} - \boldsymbol{a}_{O}$ if $k$ is an odd number and $\boldsymbol{x} = \boldsymbol{a}_{F} - \boldsymbol{a}_{O}$ if $k$ is an even number, moreover, the positive sign correspond to the case that the first gait begins from the edge $b_O$, where $\boldsymbol{x} = \boldsymbol{a}_{F} - \boldsymbol{b}_{O}$ if $k$ is an odd number and $\boldsymbol{x} = \boldsymbol{b}_{F} - \boldsymbol{b}_{O}$ if $k$ is an even number. $\boldsymbol{a}_{O}$, $\boldsymbol{b}_{O}$, $\boldsymbol{a}_{F}$, $\boldsymbol{b}_{F} \in \mathbb{R}^2$ are the coordinates of the contact vertices in $\mathcal{C}_O$ and $\mathcal{C}_F$ poses along $X$- and $Y$-axis of the frame $\{s\}$. In the optimization problems (\ref{eq:GaitPlanning}), the first constraint represents the distance of the last contact vertex ($a_F$ or $b_F$) relative to the first contact vertex ($a_O$ or $b_O$) in $X$ and $Y$ directions. The second constraint represents the relative angle between the contact edges $a_O b_O$ and $a_F b_F$, and the last constraint considers the manipulators' limitations to rotate the object.
In order to find the feasible minimum number of edge contacts, $k$, required to manipulate the object between two poses $\mathcal{C}_O$ and $\mathcal{C}_F$, we need to repeat (\ref{eq:GaitPlanning}) for different values of $k$.

\begin{figure}[!htbp]
    \centering
    \subfloat[]{\includegraphics[scale=0.5]{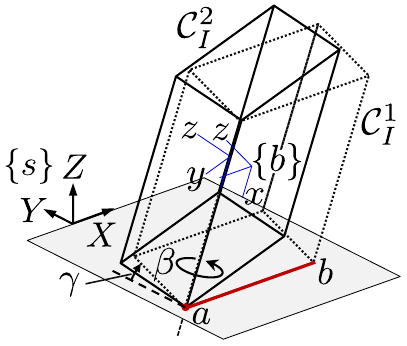}} \quad \quad
    \subfloat[]{\includegraphics[scale=0.5]{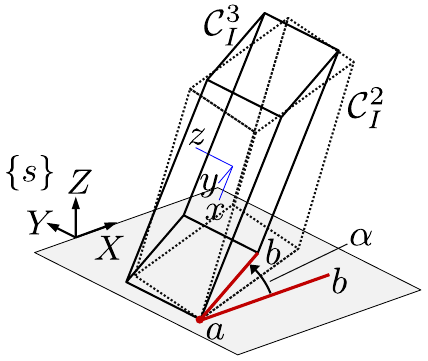}}
    \caption{Intermediate poses in object gaiting while pivoting.} 
\label{Fig:ObjectGaiting}
\end{figure}

\begin{figure}[!htbp]
    \centering
    \subfloat[]{\includegraphics[scale=0.4]{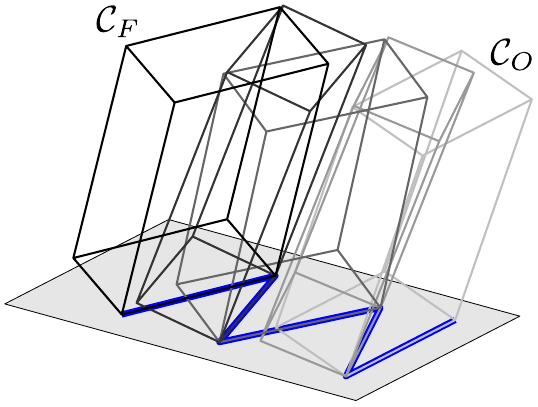}} \quad \,
    \subfloat[]{\includegraphics[scale=0.34]{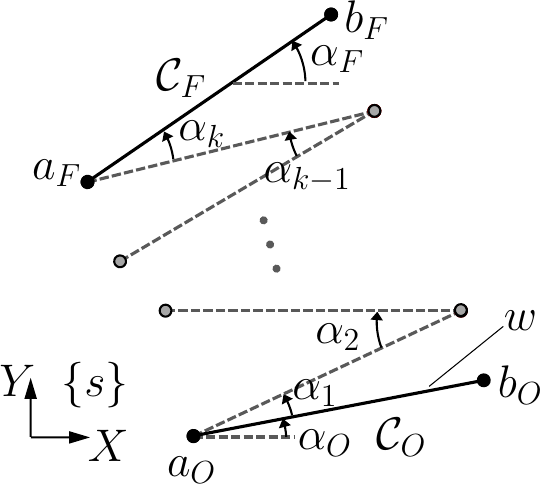}}
    \caption{A sequence of contact edges for object gaiting between two poses $\mathcal{C}_O$ and $\mathcal{C}_F$ when the first gait begins from the edge $a_O$.}
\label{Fig:IntermediateConfigs_Edges}
\end{figure}

\section{Mapping from $\mathbb{T}$-space to $\mathbb{J}$-space}
Since it is assumed that the transformation between the end-effector frame $\{e_i\}$ and contact frame $\{c_i\}$ remains constant,
after planning a path in the $\mathbb{T}$-space, we can compute the end-effector poses $\mathcal{E}_i$ for each object intermediate pose.
Then, we use the ScLERP for each of these end-effector poses individually with a shared screw parameter. To find the joint angles of the manipulators in $\mathbb{J}$-space, we use the (weighted) pseudoinverse of the manipulators' Jacobian \cite{Klein1983}.
Let $\boldsymbol{\Theta}_{t}$ and $\boldsymbol{\chi}_{t}$ be the vector of joint angles and end-effector’s pose at the step $t$, respectively.
For each manipulator, given the current end effector pose $\boldsymbol{\chi}_{t}$ and the target end effector pose $\boldsymbol{\chi}_{t+1}$ (obtained from ScLERP) we have the corresponding joint angles $\boldsymbol{\Theta}_{t+1}$ as
\begin{equation}
    \boldsymbol{\Theta}_{t+1} = \boldsymbol{\Theta}_{t} + \lambda \mathbf{J}(\boldsymbol{\Theta}_{t}) (\boldsymbol{\chi}_{t+1} - \boldsymbol{\chi}_{t}),
    \label{eq:IK}
\end{equation}
where  $0 < \lambda \le 1$ is a step length parameter (see \cite{Sarker2020} for a complete algorithm). Here $\mathbf{J}$ is the (weighted) pseudo-inverse of the manipulator Jacobian. By using (\ref{eq:IK}) between any two successive poses in $\{\mathcal{C}_O, \mathcal{C}_I, \mathcal{C}_F \}$, $\boldsymbol{\Theta}^i(j)$ ($j=1,\cdots,m$) for the $i$-th manipulator is computed. Note that we have just presented the simplest way of converting from task space to joint space. To avoid joint limits, we can augment Equation~\ref{eq:IK} with additional terms on the right hand side that belongs to the null space of the manipulator Jacobian. 

\section{Implementation and Results}
In this section, we briefly present the simulation results for manipulating a heavy cuboid object on a flat surface and over a step.
Videos of our simulations are presented in the video attachment to the paper.

\noindent
\textbf{Manipulation on a Flat Surface}: In this example, we plan motion to reorient a heavy object from an initial pose $\mathcal{C}_O$ to a final pose $\mathcal{C}_F$, in its vicinity, by object gaiting as shown in Fig.~\ref{Fig:Example_Flat}-\subref{Fig:Example_Flat_Object}.
Existing planning algorithms~\cite{yoshida2010pivoting} cannot efficiently solve this problem, because their motion plan is essentially restricted to move on Reeds and Shepp curves.
By using the proposed optimization problem (\ref{eq:GaitPlanning}), we can find the minimum number of contact edges required to manipulate the object between these two poses. The simulation results for $a_O = [0, \, 0]$, $\alpha_{O} = 0^{\circ}$, $a_F = [0.13, \, 0.13]$ m, $\alpha_F = -80^{\circ}$,  $w = 0.2$ m, $\alpha_{\text{max}} = 35^{\circ}$ are shown in Fig.~\ref{Fig:Example_Flat}-\subref{Fig:Example_Flat_Edges}. As shown, at least 3 contact edges $\bm{\alpha} = [-10.55^{\circ}, 29.56^{\circ}, -12.63^{\circ}, 27.25^{\circ}]$ (in total 10 intermediate poses) are required to reach the final pose by starting pivoting from the edge $a$.



\begin{figure}[!htbp]
    \centering
    \subfloat[]{\includegraphics[scale=0.42]{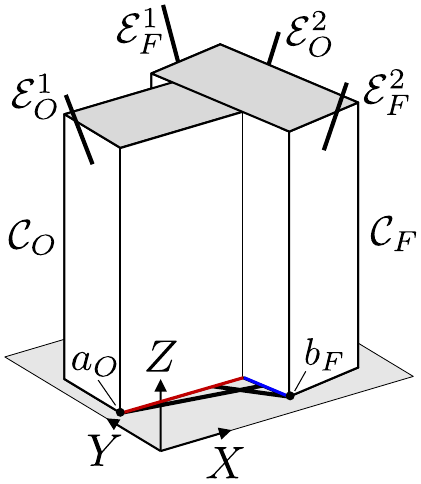}\label{Fig:Example_Flat_Object}} \qquad
    \subfloat[]{\includegraphics[scale=0.32]{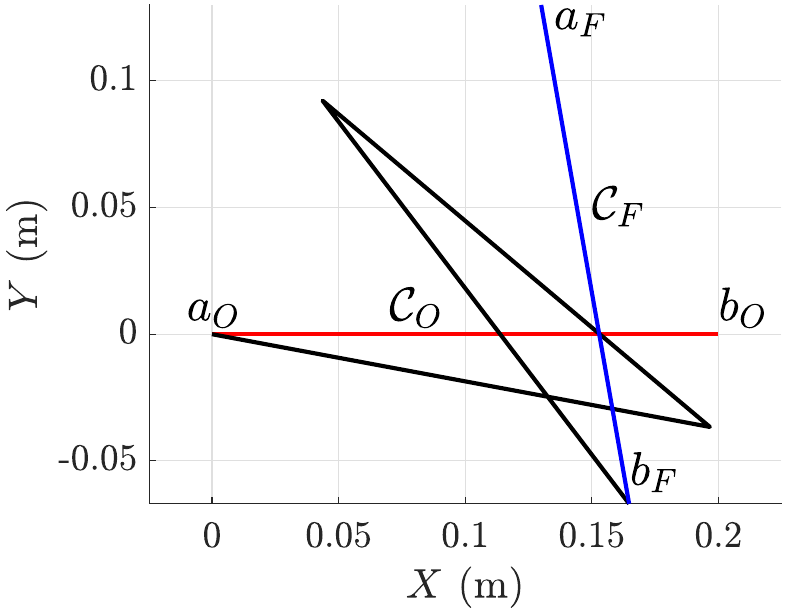}\label{Fig:Example_Flat_Edges}}
    \caption{Object gaiting on a flat surface.}
\label{Fig:Example_Flat}
\end{figure}
\noindent
\textbf{Manipulation over a Step}: In this example, we plan motion and force to manipulate a heavy object of uniform density over a step (Fig.~\ref{Fig:Example_Step}) by both 7-DoF arms of the Baxter robot. 
The computed motion plan includes 3 stages: (1) pivoting about the object edge ($\mathcal{C}_I^1$), (2) pivoting about the vertex $v$ ($\mathcal{C}_I^2$), where the object face and only the vertex $v$ are in contact with the environment, (3) changing the location of the end-effectors' contacts and pivoting about the step edge ($\mathcal{C}_F$). Thus, there are two intermediate poses $\{\mathcal{C}_I^1,\mathcal{C}_I^2\}$. For all the 3 stages, we assume that the contact locations are known beforehand. 
We implemented $\mathbb{T}$-space planning, conversion to $\mathbb{J}$-space, and force planning
to find the minimum required normal forces $f_{c_{n,1}}$ and $f_{c_{n,2}}$ at both object--end-effector contacts $\{c_1\}$ and $\{c_2\}$ in each motion stage.
The simulation parameters are given in Table~\ref{Table:cuboid_parameters}. Moreover, the friction coefficients at the manipulator contact is $0.3$ and at the environment contact is $0.4$. The optimization formulation has been implemented in MATLAB and solved using the CVX toolbox \cite{cvx} with the default solver (SDPT3) on a PC with 1.8GHz processor and 16GB RAM. In Fig.~\ref{Fig:Joint_Angles}, the variation of joint angles in the right and left arms of the Baxter robot in the $3$ stages of manipulating the object is shown. Note that as mentioned above, in stage (3), we change the end-effector contact locations before pivoting the object about the step edge. This has been reflected in the `Stage 3' of Fig.~\ref{Fig:Joint_Angles}(~\subref{Fig:right_arm}) and (~\subref{Fig:left_arm}), as the joint angle values, for all the $7$ joints of the Baxter robot, change abruptly as compared to the first two stages. In Fig.~\ref{Fig:contact_force_results}, the variations of the normal contact forces with respect to the number of iterations to reach the goal pose in the $3$ stages of object manipulation over a step are shown. In stage 1, $f_{c_{n,1}}$ and $f_{c_{n,2}}$ first decrease and become negligible at a particular object tilting angle where the weight of the object passes through its support edge, and then increases. In stage 2, since the motion is not symmetric, there is a difference between the right and left end-effector normal contact forces in order to balance the the object weight. In stage 3, the object-environment contact points are initially located closer to the object center of mass; thus, less contact forces are initially required and by pivoting the object, these forces increase.

\begin{table}[!htbp]
    \caption{Simulation Parameters for Manipulation over a Step.}
    \centering
    \setlength{\tabcolsep}{1mm}
    \renewcommand{\arraystretch}{1.2}
    \begin{tabular}{c|c}
        \hline
        \textbf{Parameter}&\textbf{Value}  \\
        \hline 
        \hline
        Weight & $20$ (N) \\
        \hline 
        Object \& Step Dimensions  & $0.15$ m $\times$ $0.15$ m $\times$ $0.1$ m, $H_s = 0.07$ (m) \\  
        \hline
    \end{tabular}
    \label{Table:cuboid_parameters}
\end{table}

\begin{figure}[!htbp]
    \centering
    \includegraphics[scale=0.47]{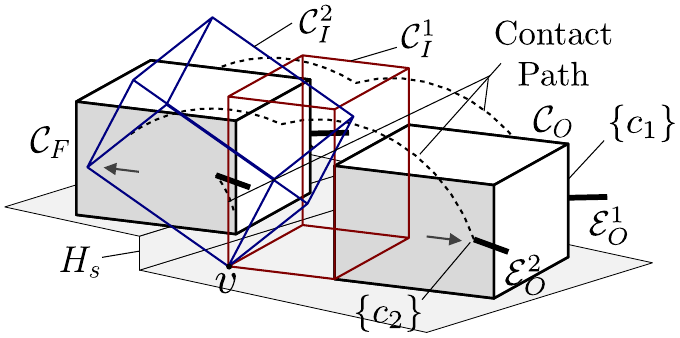}
    \caption{Object manipulation over a step.}
\label{Fig:Example_Step}
\end{figure}


\begin{figure}[!htbp]
    \centering
    \subfloat[]{\includegraphics[scale=0.25]{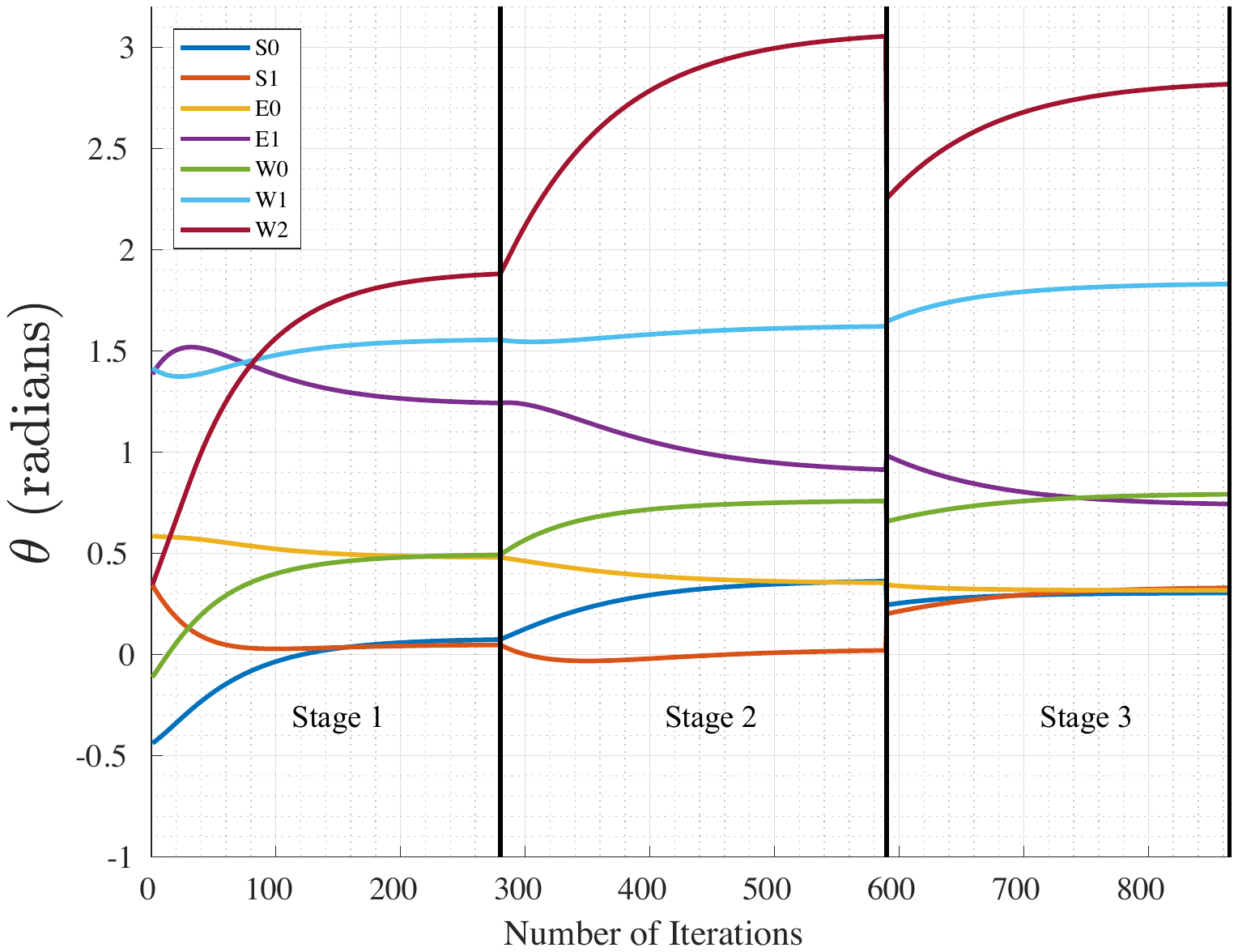}\label{Fig:right_arm}} \\
    \subfloat[]{\includegraphics[scale=0.25]{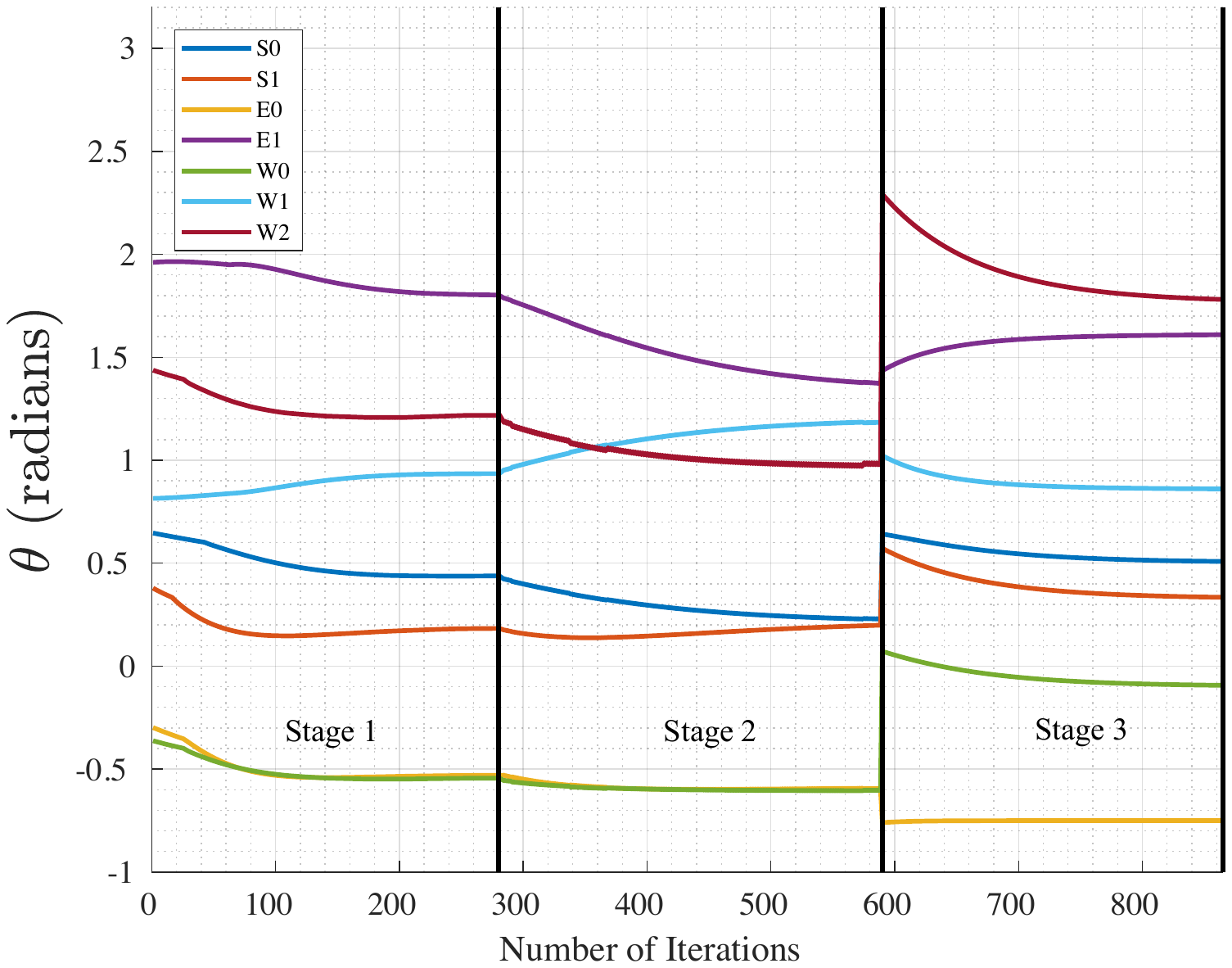}\label{Fig:left_arm}}
    \caption{Joint angle changes of the two arms, (a) right arm, (b) left arm.}
\label{Fig:Joint_Angles}
\end{figure}

\begin{figure}[!htbp]
    \centering
    \includegraphics[scale=0.30]{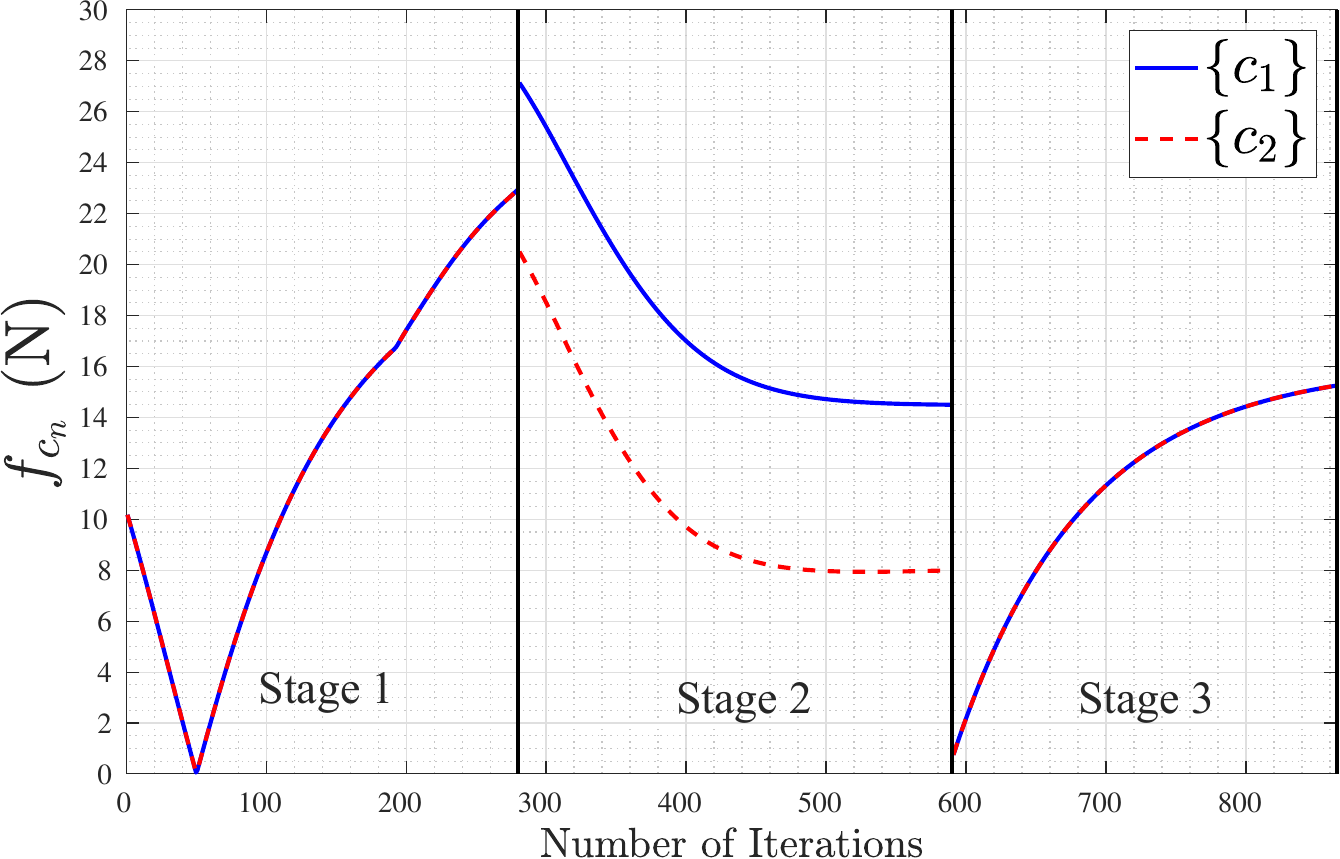}
    \caption{The normal contact forces at $\{c_1\}$ and $\{c_2\}$ where the object weight is $m = 2$kg, maximum joint torque for shoulder and elbow joints is $\tau_{\text{max}} = 50$Nm, and maximum joint torque for  wrist joints is $\tau_{\text{max}} = 15$Nm. The time taken for one iteration of force and motion planning is 1.5 seconds.}
\label{Fig:contact_force_results}
\end{figure}

\section{Conclusion and Future Work}
In this paper, we have proposed a novel approach for manipulating heavy objects while using a sequence of pivoting motions. We have implemented our proposed motion and force planning on two different scenarios; reorienting an object using gaiting and also manipulating a heavy object over a step. Given the initial and final poses of the object, we first compute the required intermediate poses. These poses can be derived by an optimization problem which computes the optimal values of the rotation angles between contact edges while \textit{object gaiting}. Then, by using ScLERP, we can interpolate between these intermediate poses while satisfying all the task-related constraints.
Using RMRC we can map the task-space based plan to the joint-space allowing us to compute the contact forces and the joint torques required to manipulate the object. Future work includes the relaxation of the quasi-static assumption for the force planning and experimental evaluation of the proposed approach.







\bibliographystyle{IEEEtran}
\bibliography{References}

\begin{thebibliography}{10}
\providecommand{\url}[1]{#1}
\csname url@samestyle\endcsname
\providecommand{\newblock}{\relax}
\providecommand{\bibinfo}[2]{#2}
\providecommand{\BIBentrySTDinterwordspacing}{\spaceskip=0pt\relax}
\providecommand{\BIBentryALTinterwordstretchfactor}{4}
\providecommand{\BIBentryALTinterwordspacing}{\spaceskip=\fontdimen2\font plus
\BIBentryALTinterwordstretchfactor\fontdimen3\font minus
  \fontdimen4\font\relax}
\providecommand{\BIBforeignlanguage}[2]{{%
\expandafter\ifx\csname l@#1\endcsname\relax
\typeout{** WARNING: IEEEtran.bst: No hyphenation pattern has been}%
\typeout{** loaded for the language `#1'. Using the pattern for}%
\typeout{** the default language instead.}%
\else
\language=\csname l@#1\endcsname
\fi
#2}}
\providecommand{\BIBdecl}{\relax}
\BIBdecl

\bibitem{PosaCT14}
M.~Posa, C.~Cantu, and R.~Tedrake, ``A direct method for trajectory
  optimization of rigid bodies through contact,'' \emph{The International
  Journal of Robotics Research (IJRR)}, vol.~33, no.~1, pp. 69--81, 2014.

\bibitem{Patankar2020}
A.~Patankar, A.~Fakhari, and N.~Chakraborty, ``Hand-object contact force
  synthesis for manipulating objects by exploiting environment,'' in
  \emph{IEEE/RSJ International Conference on Intelligent Robots and Systems
  (IROS)}, 2020.

\bibitem{BerensonSK11}
D.~Berenson, S.~Srinivasa, and J.~Kuffner, ``Task space regions: A framework
  for pose-constrained manipulation planning,'' \emph{The International Journal
  of Robotics Research}, vol.~30, no.~12, pp. 1435--1460, 2011.

\bibitem{JailletP12}
L.~{Jaillet} and J.~M. {Porta}, ``Path planning under kinematic constraints by
  rapidly exploring manifolds,'' \emph{IEEE Transactions on Robotics}, vol.~29,
  no.~1, pp. 105--117, Feb 2013.

\bibitem{Stilman10}
M.~{Stilman}, ``Global manipulation planning in robot joint space with task
  constraints,'' \emph{IEEE Transactions on Robotics}, vol.~26, no.~3, pp.
  576--584, June 2010.

\bibitem{KimU16}
B.~Kim, T.~T. Um, C.~Suh, and F.~C. Park, ``Tangent bundle rrt: A randomized
  algorithm for constrained motion planning,'' \emph{Robotica}, vol.~34, no.~1,
  p. 202–225, 2016.

\bibitem{YaoK07}
Z.~Yao and K.~Gupta, ``Path planning with general end-effector constraints,''
  \emph{Robotics and Autonomous Systems}, vol.~55, no.~4, pp. 316 -- 327, 2007.

\bibitem{bonilla2015sample}
M.~Bonilla, E.~Farnioli, L.~Pallottino, and A.~Bicchi, ``Sample-based motion
  planning for soft robot manipulators under task constraints,'' in \emph{IEEE
  International Conference on Robotics and Automation (ICRA)}, 2015.

\bibitem{KingstonMK2019}
Z.~Kingston, M.~Moll, and L.~E. Kavraki, ``Exploring implicit spaces for
  constrained sampling-based planning,'' \emph{International Journal of
  Robotics Research}, vol.~38, no. 10-11, pp. 1151--1178, 2019.

\bibitem{whitney1969resolved}
D.~E. Whitney, ``Resolved motion rate control of manipulators and human
  prostheses,'' \emph{IEEE Transactions on man-machine systems}, vol.~10,
  no.~2, pp. 47--53, 1969.

\bibitem{Pieper68}
D.~L. Pieper, ``The kinematics of manipulators under computer control,'' Ph.D.
  dissertation, Stanford University, 1968.

\bibitem{Dafle2014}
N.~C. {Dafle}, A.~{Rodriguez}, R.~{Paolini}, B.~{Tang}, S.~S. {Srinivasa},
  M.~{Erdmann}, M.~T. {Mason}, I.~{Lundberg}, H.~{Staab}, and T.~{Fuhlbrigge},
  ``Extrinsic dexterity: In-hand manipulation with external forces,'' in
  \emph{IEEE International Conference on Robotics and Automation (ICRA)}, 2014,
  pp. 1578--1585.

\bibitem{HouMason2018}
Y.~Hou, Z.~Jia, and M.~Mason, ``Fast planning for 3d any-pose-reorienting using
  pivoting,'' in \emph{IEEE International Conference on Robotics and Automation
  (ICRA)}, 2018, pp. 1631--1638.

\bibitem{HouMason2020}
------, ``{Manipulation with Shared Grasping},'' in \emph{Proceedings of
  Robotics: Science and Systems}, Corvalis, Oregon, USA, July 2020.

\bibitem{Murooka2015}
M.~{Murooka}, S.~{Nozawa}, Y.~{Kakiuchi}, K.~{Okada}, and M.~{Inaba},
  ``Whole-body pushing manipulation with contact posture planning of large and
  heavy object for humanoid robot,'' in \emph{IEEE International Conference on
  Robotics and Automation}, 2015, pp. 5682--5689.

\bibitem{Polverini2020}
M.~P. {Polverini}, A.~{Laurenzi}, E.~M. {Hoffman}, F.~{Ruscelli}, and N.~G.
  {Tsagarakis}, ``Multi-contact heavy object pushing with a centaur-type
  humanoid robot: Planning and control for a real demonstrator,'' \emph{IEEE
  Robotics and Automation Letters}, vol.~5, no.~2, pp. 859--866, 2020.

\bibitem{aiyama1993pivoting}
Y.~Aiyama, M.~Inaba, and H.~Inoue, ``Pivoting: A new method of graspless
  manipulation of object by robot fingers,'' in \emph{IEEE/RSJ International
  Conference on Intelligent Robots and Systems (IROS)}, 1993, pp. 136--143.

\bibitem{yoshida2007pivoting}
E.~Yoshida, M.~Poirier, J.-P. Laumond, R.~Alami, and K.~Yokoi, ``Pivoting based
  manipulation by humanoids: a controllability analysis,'' in \emph{IEEE/RSJ
  International Conference on Intelligent Robots and Systems (IROS)}.\hskip 1em
  plus 0.5em minus 0.4em\relax IEEE, 2007, pp. 1130--1135.

\bibitem{yoshida2008whole}
E.~Yoshida, M.~Poirier, J.-P. Laumond, O.~Kanoun, F.~Lamiraux, R.~Alami, and
  K.~Yokoi, ``Whole-body motion planning for pivoting based manipulation by
  humanoids,'' in \emph{IEEE International Conference on Robotics and
  Automation (ICRA)}, 2008, pp. 3181--3186.

\bibitem{yoshida2010pivoting}
------, ``Pivoting based manipulation by a humanoid robot,'' \emph{Autonomous
  Robots}, vol.~28, no.~1, p.~77, 2010.

\bibitem{bouyarmane2018quadratic}
K.~Bouyarmane, K.~Chappellet, J.~Vaillant, and A.~Kheddar, ``Quadratic
  programming for multirobot and task-space force control,'' \emph{IEEE
  Transactions on Robotics}, vol.~35, no.~1, pp. 64--77, 2018.

\bibitem{khatib1995inertial}
O.~Khatib, ``Inertial properties in robotic manipulation: An object-level
  framework,'' \emph{The international journal of robotics research}, vol.~14,
  no.~1, pp. 19--36, 1995.

\bibitem{Sarker2020}
A.~Sarker, A.~Sinha, and N.~Chakraborty, ``On screw linear interpolation for
  point-to-point path planning,'' in \emph{IEEE/RSJ International Conference on
  Intelligent Robots and Systems (IROS)}, 2020.

\bibitem{Klein1983}
C.~A. {Klein} and C.~{Huang}, ``Review of pseudoinverse control for use with
  kinematically redundant manipulators,'' \emph{IEEE Transactions on Systems,
  Man, and Cybernetics}, vol. SMC-13, no.~2, pp. 245--250, 1983.

\bibitem{cvx}
M.~Grant and S.~Boyd, ``{CVX}: Matlab software for disciplined convex
  programming, version 2.1,'' \url{http://cvxr.com/cvx}, Mar. 2014.

\end{thebibliography}

\end{document}